\documentclass[11pt]{article}
\usepackage{fullpage}
\usepackage{amsmath,amssymb,amsthm,mathrsfs,pgf,tikz,caption,subcaption,mathtools}
\usepackage[ruled,vlined,linesnumbered,noresetcount]{algorithm2e}
\usepackage{natbib}
\usepackage{amsmath,amssymb,amsthm,mathtools}
\usepackage[utf8]{inputenc} % allow utf-8 input
\usepackage[T1]{fontenc}    % use 8-bit T1 fonts
\usepackage{hyperref}       % hyperlinks
\usepackage{url}            % simple URL typesetting
\usepackage{booktabs}       % professional-quality tables
\usepackage{amsfonts}       % blackboard math symbols
\usepackage{nicefrac}       % compact symbols for 1/2, etc.
\usepackage{microtype}      % microtypography
\usepackage{bbm}
\usepackage{enumerate}
\usepackage{enumitem}
\usepackage{algorithmic}

\newtheorem{theorem}{Theorem}[section]
\newtheorem{lemma}[theorem]{Lemma}

\newtheorem{proposition}[theorem]{Proposition}

\newtheorem{definition}{Definition}[section]
\newtheorem{remark}{Remark}
\newcommand{\Scal}{\mathcal{S}}
\newcommand{\Acal}{\mathcal{A}}
\newcommand{\Fcal}{\mathcal{F}}

\newcommand{\say}[1]{``#1''} 
\DeclareMathOperator*{\argmax}{\arg\!\max}
\DeclareMathOperator*{\argmin}{\arg\!\min}
\DeclareMathOperator{\poly}{poly}

\newcommand{\gapmin}{\rho}

\newcommand{\mdp}{\mathcal{M}}

\newcommand{\idxn}[1][]{\ifthenelse{\equal{#1}{}}{\mathsf{INDQ}_n}{\mathsf{INDQ}_{#1}}}

\newcommand{\simplex}{\triangle}

\newcommand{\expect}{\mathbb{E}}

\newcommand{\states}{\mathcal{S}}
\newcommand{\trans}{P}

\newcommand{\actions}{\mathcal{A}}

\usepackage{xcolor}

\title{Agnostic $Q$-learning with Function Approximation in Deterministic Systems: Tight Bounds on Approximation Error and Sample Complexity}
\date{}
\author{
Simon S. Du\thanks{Institute for Advanced Study. Email: \texttt{ssdu@ias.edu}}
\and Jason D. Lee\thanks{Princeton University. Email: \texttt{jasonlee@princeton.edu}}
\and Gaurav Mahajan\thanks{University of California, San Diego. Email: \texttt{gmahajan@eng.ucsd.edu}}
\and Ruosong Wang\thanks{Carnegie Mellon University. Email:\texttt{ruosongw@andrew.cmu.edu}}
}

\begin{document}
\maketitle



\begin{abstract}
%We consider the problem of agnostic reinforcement learning over deterministic system with function approximation and show a tight connection between efficient learning, the \say{optimality gap} $\rho$ and the \say{agnostic error} $\delta$. In particular, when $\rho = \Omega(d \cdot \delta)$, where $d$ is the eluder dimension, then we can efficiently solve the problem of agnostic reinforcement learning. In conjunction with previous results for linear function approximation, this shows that the connection is tight for linear function approximation.

%\citet{wen2013efficient} demonstrated that in deterministic systems, if one uses function approximation for $Q$-learning, then one only needs $Hd$ trajectories to find the optimal policy where $H$ is the planning horizon and $d$ is the Eluder dimension of a pre-specified function class $\mathcal{F}$.
%The crucial assumption in their result is that the optimal $Q$-function, $Q^*$ lies in $\mathcal{F}$.

The current paper studies the problem of agnostic $Q$-learning with function approximation in deterministic systems where the optimal $Q$-function is approximable by a function in the class $\mathcal{F}$ with approximation error $\delta \ge 0$. We propose a novel recursion-based algorithm and show that if  $\delta = O\left(\rho/\sqrt{\dim_E}\right)$, then one can find the optimal policy using $O\left(\dim_E\right)$ trajectories, where $\rho$ is the gap between the optimal $Q$-value of the best actions and that of the second-best actions and $\dim_E$ is the Eluder dimension of $\mathcal{F}$. Our result has two implications:
\begin{enumerate}
\item In conjunction with the lower bound in [Du et al., ICLR 2020], our upper bound suggests that the condition $\delta = \widetilde{\Theta}\left(\rho/\sqrt{\dim_E}\right)$ is necessary and sufficient for algorithms with polynomial sample complexity.
\item In conjunction with the lower bound in [Wen and Van Roy, NIPS 2013], our upper bound suggests that the sample complexity $\widetilde{\Theta}\left(\dim_E\right)$ is tight even in the agnostic setting.
\end{enumerate}
Therefore, we settle the open problem on agnostic $Q$-learning proposed in [Wen and Van Roy, NIPS 2013].
We further extend our algorithm to the stochastic reward setting and obtain similar results.

\end{abstract}

\section{Introduction}
% !TEX root = main.tex
$Q$-learning is a fundamental approach in reinforcement learning~\citep{watkins1992q}.
Empirically, combining $Q$-learning with function approximation schemes has lead to tremendous success on various sequential decision-making problems.
However, theoretically, we only have a good understanding of $Q$-learning in the tabular setting.
\citet{strehl2006pac, jin2018q} show that with certain exploration techniques, $Q$-learning provably finds a near-optimal policy with sample complexity polynomial in the number of states, number of actions and the planning horizon.
However, modern reinforcement learning applications often require dealing with huge state space where the polynomial dependency on the number of states is not acceptable.
%Despite providing valuable theoretical insights, these algorithms do not address problems with huge state space common in modern reinforcement learning applications.

%Recently, there has been great interest in designing and analyzing $Q$-learning algorithms with linear function approximation~\cite{wen2013efficient, yang2019sample, yang2019reinforcement, du2019provably, jin2019provably}.
Recently, there has been great interest in designing and analyzing $Q$-learning algorithms with linear function approximation~\citep{wen2013efficient,du2019provably}.
Under various additional assumptions, these works show that one can obtain a near-optimal policy using $Q$-learning with sample complexity polynomial in the feature dimension $d$ and the planning horizon, if the optimal $Q$-function is an exact linear function of the $d$-dimensional features of the state-action pairs.

%agnostic
A major drawback of these works is that the algorithms can only be applied in the well-specified case, i.e., the optimal $Q$-function is an exact linear function.
In practice, the optimal $Q$-function is usually linear up to small approximation errors instead of being exactly linear.
In this paper, we focus on the agnostic setting, i.e., the optimal $Q$-function can only be approximated by a function class with approximation error $\delta$, which is closer to practical scenarios.
Indeed, designing a provably efficient $Q$-learning algorithm in the agnostic setting is an open problem posed by \citet{wen2013efficient}.

Technically, the agnostic setting is arguably more challenging than the exact setting.
As recently shown by~\citet{du2019good}, for the class of linear functions, when the approximation error $\delta = \Omega(\sqrt{\poly(H) / d})$ where $H$ is the planning horizon, any algorithm needs to sample exponential number of trajectories to find a near-optimal policy even in deterministic systems.
Therefore, for algorithms with polynomial sample complexity, additional assumptions are needed to bypass the hardness result.
For the exact setting $\delta = 0$, \citet{wen2013efficient}~show that one can find an optimal policy using polynomial number of trajectories for linear functions in deterministic systems, which implies that the agnostic setting could be exponentially harder than the exact setting.

Due to the technical challenges, for the agnostic setting, previous papers mostly focus on the bandit setting or reinforcement learning with a  generative model~\citep{lattimore2019learning, van2019comments, neu2020efficient}, and much less is known for the standard reinforcement learning setting.
In this paper, we design $Q$-learning algorithms with provable guarantees in the agnostic case for the standard reinforcement learning setting.

\subsection{Our Contributions}
Our main contribution is a provably efficient $Q$-learning algorithm for the agnostic setting with general function approximation in deterministic systems.
Our result settles the open problem posed by \citet{wen2013efficient}.

\begin{theorem}[Informal]\label{thm:informal}
For a given episodic deterministic system and a function class $\Fcal$, suppose there exists $f \in \Fcal$ such that the optimal $Q$-function $Q^*$ satisfies
\[
|f(s, a) - Q^*(s, a)| \le \delta
\]
for any state-action pair $(s, a)$.
Suppose $\gapmin = \Omega(\sqrt{\dim_E} \delta)$, where the optimality gap $\gapmin$ is the gap between the optimal $Q$-value of the best action and that of the second-best action (formally defined in Definition~\ref{def:gap}) and $\dim_E$ is the Eluder dimension of $\mathcal{F}$ (see Definition~\ref{def:eluder}), our algorithm finds the optimal policy using $O(\dim_E)$ trajectories.
\end{theorem}

Our main assumption in Theorem~\ref{thm:informal} is that the optimality gap $\gapmin$ satisfies $\gapmin = \Omega(\sqrt{\dim_E} \delta)$.
Below we discuss the necessity of this assumption and its connection with the recent hardness result in~\citep{du2019good}.

in~\citep{du2019good}, it has been proved that in deterministic systems, if the optimal $Q$-function can be approximated by linear functions with approximation error $\delta = \Omega(\sqrt{\mathrm{poly}(H) / d})$, any algorithm needs to sample exponential number of trajectories to find a near-optimal policy even in deterministic systems, where $d$ is the input dimension for the linear functions. Using the same technique as in~\citep{du2019good}, in the supplementary material we show the following hardness result for $Q$-learning with linear function approximation in the agnostic setting.
\begin{proposition}[Generalization of Theorem 4.1 in~\citep{du2019good}]\label{thm:informal_lb}
There exists a family of deterministic systems that the optimal $Q$-function can be approximated by linear functions with approximation error $\delta = \Omega(\sqrt{\poly(C) / d})$ and the optimality gap $\gapmin = 1$, such that any algorithm that returns a $1/2$-optimal policy needs to sample $\Omega(2^C)$ trajectories.
\end{proposition}

By setting $C = O(\log(Hd))$ such that $2^C = \poly(Hd)$, Theorem~\ref{thm:informal_lb} implies that for any algorithm with polynomial sample complexity, if $\gapmin = 1$, then the approximation error $\delta$ that can be handled by the algorithm is at most $\widetilde{O}(\sqrt{1 / d})$.
Since the Eluder dimension of linear functions is $\widetilde{O}(d)$, the condition $\gapmin = \Omega(\sqrt{\dim_E} \delta)$ in our algorithm is tight up to logarithm factors and can not be significantly improved in the worst case.

%{\color{red} Ruosong: Shall we put is here?}\simon{I think it's fine}
One interpretation of the hardness result in~\citep{du2019good} is that in the worst case, there is an $\widetilde{\Omega}(\sqrt{d})$-error amplification phenomenon in reinforcement learning with linear function approximation. 
Our algorithm in Theorem~\ref{thm:informal} complements this hardness result by showing that there is an algorithm with error amplification factor at most $\widetilde{O}(\sqrt{d})$, and thus both results are tight up to logarithm factors.
From this point of view, our result is in the same spirit as the results in~\citep{lattimore2019learning, van2019comments}, which also demonstrate the tightness of the hardness result in~\citep{du2019good}.
However, as will be made clear, technically our result significantly deviates from those in~\citep{lattimore2019learning, van2019comments}.
See Section~\ref{sec:related} for more detailed comparison with~\citep{lattimore2019learning, van2019comments}.

Note that the sample complexity of our algorithm is linear in the Eluder dimension of the function class.
In conjunction with the lower bound in~\citep{wen2013efficient} which holds in the exact setting in deterministic systems, our algorithm shows that $\widetilde{\Theta}(\dim_E)$ sample complexity is tight even in the agnostic setting.
Another interesting aspect of Theorem~\ref{thm:informal} is that the sample complexity of our algorithm does not depend on the size of the action space $|A|$, which potentially makes the algorithm more practical since the action space can be huge or even continuous in certain applications.
%Moreover, to make the running time of our algorithm independent of the size of the action space $|A|$, our algorithm relies on access to a certain oracle called Maximum Uncertainty Oracle, which will be defined and discussed later in Section~\ref{sec:general}.

Finally, we show how to generalize our results to handle stochastic rewards. Under the same assumption that $\gapmin = \Omega(\sqrt{\dim_E} \delta)$, our algorithm finds an optimal policy using $\frac{\poly(\dim_E, H)}{\gapmin^2}\log(1/p)$ trajectories with failure probability $p$.
We would like to remark that the $\log(1/p)/\gapmin^2$ dependency is necessary for finding optimal policies even in the bandit setting~\citep{mannor2004sample}. 
\subsection{Organization}
In Section \ref{sec:related}, we review related work. 
In Section \ref{sec:prelims}, we introduce necessary notations, definitions and assumptions. 
In Section \ref{sec:linear}, we discuss the special case where $\Fcal$ is the class of linear functions to demonstrate the high-level approach of our algorithm and the intuition behind the analysis.
We then present the result for general function classes in Section \ref{sec:general}.  We conclude and discuss future work in Section \ref{sec:discussion}.

\section{Related Work}
\label{sec:related}
% !TEX root = main.tex
Classical theoretical reinforcement learning literature studies asymptotic behavior of concrete algorithms or finite sample complexity bounds for $Q$-learning algorithms under various assumptions~\citep{melo2007q, zou2019finite}.
These works usually assume the initial policy has certain benign properties, which may not hold in practical applications.
Another line of work focuses on sample complexity and regret bound in the tabular setting~\citep{lattimore2012pac, azar2013minimax, sidford2018near, sidford2018variance,agarwal2019optimality,jaksch2010near,agrawal2017posterior,azar2017minimax, kakade2018variance}, for which exploration becomes much easier.
\citet{strehl2006pac,jin2018q} show that with certain exploration techniques, $Q$-learning provably finds a near-optimal with polynomial sample complexity.
However, these works have sample complexity at least linearly depends on the number of states, which is necessary without additional assumptions~\citep{jaksch2010near}.

Various exploration algorithms are proposed for $Q$-learning with function approximation~\citep{azizzadenesheli2018efficient, fortunato2018noisy, lipton2018bbq,osband2016generalization,pazis2013pac}.
However, none of these algorithms have polynomial sample complexity guarantees.
\citet{li2011knows} propose a $Q$-learning algorithm which requires the Know-What-It-Knows oracle. However, it is unknown how to implement such oracle in general.
\citet{wen2013efficient} propose an algorithm for $Q$-learning with function approximation in deterministic systems which works for a family of function classes in the exact setting.
For the agnostic setting, the algorithm in~\citep{wen2013efficient} can only be applied to a special case called ``state aggregation case''.
See Section 4.3 in~\citep{wen2013efficient} for more details.
Indeed, as stated in the conclusion of \citep{wen2013efficient}, designing provably efficient algorithm for agnostic $Q$-learning with general function approximation is a challenging open problem.

Using the distribution shift checking oracle, \citet{du2019provably} propose an algorithm for $Q$-learning with linear function approximation in the exact setting.
The algorithm in~\citep{du2019provably} further requires conditions on the optimality gap $\rho$ and a low-variance condition on the transition.
Our algorithms also requires conditions on the optimality gap $\rho$ and shares similar recursion-based structures as the algorithm in~\citep{du2019provably}.
However, our algorithm handles general function classes with bounded Eluder dimension and with approximation error, neither of which can be handled by the algorithm in~\citep{du2019provably}.

Recently, \citet{du2019good} proved lower bounds for $Q$-learning algorithm in the agnostic setting.
As mentioned in the introduction, our algorithm complements the lower bounds in~\citep{du2019good} and demonstrates the tightness of their lower bound.
\citet{lattimore2019learning, van2019comments} also give algorithms in the agnostic setting to demonstrate the tightness of the lower bound in~\citep{du2019good} from other perspectives.
Technically, our results are different from those in~\citep{lattimore2019learning, van2019comments} in the following ways.
First, we study the standard reinforcement learning setting, where \citet{van2019comments} focus on the bandit setting and \citet{lattimore2019learning} study both the bandit setting and reinforcement learning with a generative model.
Second, for the reinforcement learning result in~\citep{lattimore2019learning}, it is further assumed that $Q$-functions induced by {\em all} polices can be approximated by linear functions, while in this paper we only assume the optimal $Q$-function can be approximated by a function class with bounded Eluder dimension, which is much weaker than the assumption in~\citep{lattimore2019learning}.
Finally, in this paper, we focus on finding the optimal policy instead of a near-optimal policy, and thus it is necessary to put assumptions on the optimality gap $\rho$.
In conjunction with the lower bound in~\citep{du2019good}, we give a tight condition $\delta = \widetilde{\Theta}\left(\rho/\sqrt{\dim_E}\right)$ under which there is an algorithm with polynomial sample complexity to find the optimal policy.
On the other hand, the algorithm in~\citep{lattimore2019learning} does not require conditions on the optimality gap $\rho$ and thus can only find near-optimal policies. 
Their result demonstrates the tightness of the hardness result in~\citep{du2019good} from another perspective by giving a tight bound on the suboptimality of the policy found by the algorithm and the approximation error $\delta$.

Recently, a line of work study $Q$-learning in the linear MDP setting~\citep{yang2019sample, yang2019reinforcement, jin2019provably, wang2019optimism}.
In the linear MDP setting, it is assumed that both the reward function and the transition operator is linear, which is stronger than the assumption that the optimal $Q$-function is linear studied in this paper.
For the linear MDP setting, algorithms with polynomial sample complexity are known, and these algorithms can usually handle approximation errors on the reward function and the transition operator.

\section{Preliminaries}
\label{sec:prelims}
% !TEX root = main.tex
\subsection{Notations} We begin by introducing necessary notations. We write $[n]$ to denote the set $\{1,2,\ldots, n\}$. We use $\|\cdot\|_p$ to denote the $\ell_p$ norm of a vector. For any finite set $S$, we write $\simplex(S)$ to denote the probability simplex. 

\subsection{Episodic Reinforcement Learning}
\label{sec:mdp}
In this paper, we consider Markov Decision Processes with deterministic transition and stochastic reward.
Formally, let $\mdp =\left(\states, \actions, H,\trans,R \right)$ be a Markov Decision Process (MDP) where $\states$ is the state space, $\actions$ is the action space, $H \in \mathbb{Z}_+$ is the planning horizon, $\trans: \states \times \actions \rightarrow \states$ is the deterministic transition function which takes a state-action pair and returns a state, and $R : \states \times \actions \rightarrow \triangle \left( \mathbb{R} \right)$ is the reward distribution. 
When the reward is deterministic, we may regard $R : \states \times \actions \rightarrow  \mathbb{R} $ as a function instead of a distribution.
We assume there is a fixed initial state $s_1$.

A policy $\pi: \states \rightarrow \simplex(\actions)$ prescribes a distribution over actions for each state.
%We let the fixed start state $s_1$.
%\footnote{
%	Some papers assume the starting	state  is sampled from a distribution $P_1$.
%	Note this is equivalent to assuming a fixed state $s_1$, by setting $\trans(s_1,a) = P_1$ for all $a \in \actions$ and now our $s_2$ is equivalent to the starting state in their assumption.
%}
The policy $\pi$ induces a (random) trajectory $s_1,a_1,r_1,s_2,a_2,r_2,\ldots,s_H,a_H,r_H$
where $a_1 \sim \pi(s_1)$, $r_1 \sim R(s_1,a_1)$, $s_2 = \trans(s_1,a_1)$, $a_2 \sim \pi(s_2)$, etc.
To streamline our analysis, for each $h \in [H]$, we use $\states_h \subseteq \states$ to denote the set of states at level $h$, and we assume $\states_h$ do not intersect with each other.
We also assume $\sum_{h = 1}^{H}r_h \in [0, 1]$.
Our goal is to find a policy $\pi$ that maximizes the expected total reward $\expect \left[\sum_{h=1}^{H} r_h \mid \pi\right]$. 
We use $\pi^*$ to denote the optimal policy.
%We say a policy $\pi$ is $\varepsilon$-optimal if $\expect \left[\sum_{h=0}^{H - 1} r_h\mid \pi\right] \ge \expect \left[\sum_{h=0}^{H - 1} r_h\mid \pi^*\right] - \varepsilon$.

%In this paper we prove lower bounds for deterministic systems, i.e., MDPs with deterministic transition $P$, deterministic reward $R$.
%In this setting, $P$ and $R$ can be regarded as functions instead of distributions. 
%Since deterministic systems are special cases of general stochastic MDPs, lower bounds proved in this paper still hold for more general MDPs.

%For a given policy $\pi$, we use $\dist^\pi_h$ to denote the distribution over $\states_h$ induced by executing policy $\pi$.

\subsection{$Q$-function, $V$-function and the Optimality Gap}
An important concept in RL is the $Q$-function.
Given a policy $\pi$, a level $h \in [H]$ and a state-action pair $(s,a) \in \states_h \times \actions$, the $Q$-function is defined as $Q_h^\pi(s,a) = \expect\left[\sum_{h' = h}^{H}r_{h'}\mid s_h =s, a_h = a, \pi\right]$.
%It will also be useful to define the value function of a given state $s \in \states_h$ as $V_h^\pi(s)=\expect\left[\sum_{h' = h}^{H}r_{h'}\mid s_h =s, \pi\right]$.
For simplicity, we denote $Q_h^*(s,a) = Q_h^{\pi^*}(s,a)$.
It will also be useful to define the value function of a given state $s \in \states_h$ as $V_h^\pi(s)=\expect\left[\sum_{h' = h}^{H}r_{h'}\mid s_h =s, \pi\right]$.
For simplicity, we denote $V_h^*(s)=V_h^{\pi^*}(s)$.
Throughout the paper, for the $Q$-function $Q_h^{\pi}$ and $Q_h^*$ and the value function $V_h^\pi$ and $V_h^*$, we may omit $h$ from the subscript when it is clear from the context.
% and $V_h^*=V_h^{\pi^*}(s)$.
%Recall that if we know $Q_h^*$, we can just choose the action greedily to obtain the optimal policy, i.e., $\pi^*(s) = \argmax_{a \in \actions} Q_h^*(s,a)$.
%\lin{Keep an eye on $\pi$, it may need $h$.}
%By $Q$-learning, we mean we try to obtain a good estimate of the optimal $Q$-function $Q^*$ so that we can just choose the action greedily according to the estimated $Q$-function.
%In this paper, we make the following assumption about the variation of the suboptimality of policies~\cite{everitt2002cambridge}.
%\begin{asmp}[Bounded Coefficient of Variation of Policy Sub-optimality]\label{asmp:var_bound}
%There exists a constant $1 \le C< \infty$, such that for any fixed level $h \in [H]$ and deterministic policy $\pi$, \begin{align*}
%\expect_{s \sim \dist_h^\pi}\left[\abs{V^\pi(s)-V^{*}(s)}^2\right] \le C \left(\expect_{s \sim \dist_h^\pi}\left[\abs{V^\pi(s)-V^{*}(s)}\right]\right)^2.
%\end{align*}
%\end{asmp}
%Intuitively, this assumption says the variation due to the randomness over states is not too large comparing to the mean.
%For example, if the transition is deterministic, then this assumption holds with $C=1$.

In addition to these definitions, we list below an important concept, the optimality gap, which is widely used in reinforcement learning and bandit literature.
\begin{definition}[Optimality Gap]
	\label{def:gap}
	The optimality gap $\gapmin$ is defined as \[\rho = \inf_{Q^*(s, a) \neq V^*(s)} V^*(s) - Q^*(s, a).\]
\end{definition}
%\begin{asmp}[Optimality Gap]
%	\label{asmp:weak_gap}
%	There exists $\gapmin > 0$ such that for any $s \in \states$ and $a \in \actions$, either $Q^*(s, a) = \max_{a'} Q^*(s, a')$ or $Q^*(s, a) \le \max_{a'} Q^*(s, a') - \gapmin$.
%\end{asmp}
	In words, $\gapmin$ is the smallest reward-to-go difference between the best set of actions and the rest.
	Recently, \citet{du2019provably} gave a provably efficient $Q$-learning algorithm based on this assumption and \citet{simchowitz2019non} showed that with this condition, the agent only incurs logarithmic regret in the tabular setting. 

\subsection{Function Approximation and Eluder Dimension}
When the state space is large, we need structures on the state space so that reinforcement learning methods can generalize. 
For a given function class $\Fcal$, each $f \in \Fcal$ is a function that maps a state-action pair to a real number.
For a given MDP and a function class $\Fcal$, we define the approximation error to the optimal $Q$-function as follow.
%We constrain the optimal Q-function to a pre-specified function class Q [7], e.g., the class of linear functions. In this paper we associate each h ∈ [H] and a ∈ A with a Q-function fha ∈ Q.

\begin{definition}[Approximation Error]\label{def:approx_error}
	For a given MDP and a function class $\Fcal$, the approximation error $\delta$ is defined to be
	\[
	\delta = \inf_{f \in \Fcal} \sup_{(s, a) \in \states \times \actions} |f(s, a) - Q^*(s, a)|,
	\]
	where $Q^* : \states \times \actions \to \mathbb{R}$ is the optimal $Q$-function of the MDP.
\end{definition}

Here, the approximation error $\delta$ characterizes how well the given function class $\Fcal$ approximates the optimal $Q$-function.
When $\delta = 0$, then optimal $Q$-function can be perfectly predicted by the function class, which has been studied in previous papers~\citep{wen2013efficient, du2019provably}.
In this paper, we focus the case $\delta > 0$.

An important function class is the class of linear functions.
We assume the agent is given a feature extractor $\phi : \states \times \actions \to \mathbb{R}^d$ where $\|\phi(s, a)\|_2 \le 1$ for all state-action pairs.
Here, the feature extractor can be hand-crafted or a pre-trained neural network that transforms a state-action pair to a $d$-dimensional embedding.
Given the feature extractor $\phi$, we define the class of linear functions as follow.
\begin{definition}\label{def:linear_function}
For a vector $\theta \in \mathbb{R}^d$, we define 
\[
f_{\theta}(s, a) = \theta^{\top}\phi(s, a).
\]
The class of linear functions is defined as
\[
\Fcal = \{f_{\theta} \mid \|\theta\|_2 \le 1 \}.
\]
\end{definition}
Here we assume $\|\theta\|_2 \le 1$ only for normalization purposes. 

For general function classes, an important concept is the {\em Eluder dimension}, for which we first need to introduce the concept of $\epsilon$-dependence.
\begin{definition}[$\epsilon$-dependence~\citep{russo2013eluder}]
\label{def:ind}
For a function class $\Fcal$,
we say a state-action pair $(s,a)$ is $\epsilon$-dependent on state-action pairs $\{(s_1, a_1), \ldots, (s_n, a_n)\} \subset \Scal \times \Acal$ with respect to $\Fcal$ if for all $f_1, f_2 \in \Fcal$,
\begin{align*}
\sum_{i=1}^n |f_1(s_i, a_i) - f_2(s_i, a_i)|^2 \leq \epsilon^2
\implies |f_1(s,a) - f_2(s,a)|^2 \leq \epsilon^2.
\end{align*}Further, $(s,a)$ is $\epsilon$-independent of state-action pairs $\{(s_1,a_1),\ldots, (s_n,a_n)\}$ if $(s,a)$ is not $\epsilon$-dependent on state-action pairs $\{(s_1,a_1),\ldots, (s_n,a_n)\}$.
\end{definition}
Now, we recall the definition of $\epsilon$-Eluder dimension as introduced in \citet{russo2013eluder}.
\begin{definition}[Eluder Dimension]
\label{def:eluder}
For a function class $\Fcal$,
the $\epsilon$-Eluder dimension $\dim_E(\Fcal, \epsilon)$ is the length of the longest sequence of elements in $\Scal \times \Acal$ such that every element is $\epsilon'$-independent of its predecessors for some $\epsilon'\geq \epsilon$.
\end{definition}

%\gnote{Fix this.}
%Note when $\epsilon=0$, this definition recovers the definition used in \cite{wen2013efficient}.
As an example, when $\Fcal$ is the class of linear functions with norm $\|\theta\|_2\leq 1$ and $\|\phi(s,a)\|_2\leq 1$, the Eluder dimension $\dim_E(\Fcal, \epsilon)$ is $O(d \log (1/\epsilon))$ as noted in Example 4 in~\citet{russo2013eluder}.
We refer interested readers to~\citet{russo2013eluder} for more examples.

\section{Algorithm for Linear Functions}
\label{sec:linear}
% !TEX root = main.tex
In this section, we consider the special case where $\Fcal$ is the class of linear functions to demonstrate the high-level approach of our algorithm and the intuition behind the analysis.
For simplicity, we also assume that the size of action space $\actions$ is bounded by a constant and the reward is deterministic.
We show how to remove these assumptions in the following sections.

\subsection{Algorithm and High-level Intuition}
\label{sec:linear-high}
In this section we present the description of our algorithm.
Our algorithm is divided into two parts: Algorithm~\ref{alg:main} in which we define the main loop and Algorithm~\ref{alg:explore} in which we define a recursion-based subroutine $\mathsf{Explore}(s)$ to calculate the optimal values.
Intuitively, the subroutine $\mathsf{Explore}(s)$ should return $V^*(s)$, and upon the termination of $\mathsf{Explore}(s)$ we should have $\pi(s) = \pi^*(s)$.
These properties will be proved formally in Section~\ref{sec:analysis_linear}.

In our algorithm, we maintain a dataset to store the features of a subset of the state-action pairs $\phi(s, a)$ and their optimal $Q$-values $Q^*(s, a)$.
Here, the matrix $C \in \mathbb{R}^d$ is the covariance of the dataset, i.e., $C = \sum \phi(s, a) \phi(s, a)^{\top}$ and $Y = \sum \phi(s, a) Q^*(s, a)$. 
In order to predict the optimal $Q$-value of an unseen state-action pair $(s, a)$ using least squares, we may directly calculate $\phi(s, a)^{\top}C^{-1}Y$ if $C$ is invertible.
We use a ridge term of $\gapmin^2/16$ to make sure $C$ is always invertible.

The high-level idea behind our algorithm is simple: we use least squares to predict the optimal $Q$-value whenever possible, and use recursions to figure out the optimal $Q$-value otherwise.
One technical subtlety here is that what condition we should check to decide whether we can calculate the optimal $Q$-value directly by least squares or we need to make recursive calls.
This condition needs to be chosen carefully, since if we make too many recursive calls, the overall sample complexity will be unbounded, and if we make too few recursive calls, the optimal $Q$-values estimated by linear squares will be inaccurate which affects the correctness of the algorithm.

In Line~\ref{line:dsec} of $\mathsf{Explore}(s)$, we check whether $\phi(s, a)^{\top} C^{-1} \phi(s, a) \le 1$, which is the condition we use to decide whether we should make recursive calls or calculate the optimal $Q$-value directly by least squares.
Here $\phi(s, a)^{\top} C^{-1} \phi(s, a)$ is the variance of the prediction, which is common in UCB-type algorithm for linear contextual bandit (see e.g.~\cite{li2010contextual}).
In our algorithm, instead of using $\phi(s, a)^{\top} C^{-1} \phi(s, a)$ as an uncertainty bonus, we directly check its magnitude to decide whether the linear predictor learned on the collected dataset generalizes well on the new data $\phi(s, a)$ or not.
The effectiveness of such a choice will made clear in the formal analysis given in Section~\ref{sec:analysis_linear}.
Moreover, in order to make sure that the value returned by $\mathsf{Explore}(s)$ is accurate, in Line~\ref{calc_V} of $\mathsf{Explore}(s)$, we make recursive calls instead of using the estimated $Q$-values $\hat{Q}$.
As will be shown in Section~\ref{sec:analysis_linear}, such a choice guarantees that the value returned by $\mathsf{Explore}(s)$ always equals $V^*(s)$.
\begin{algorithm}[tb]
	\caption{Main Algorithm}
	\label{alg:main}
	\begin{algorithmic}[1]
		\STATE Initialize the current policy $\pi$ arbitrarily
		\STATE {\bfseries set} $C =  \gapmin^2/16 \cdot I \in \mathbb{R}^{d \times d}$
		\STATE {\bfseries set} $Y = 0 \in \mathbb{R}^d$
		%\FOR{$h=1$ {\bfseries to} $H - 1$}
		\STATE {\bfseries invoke} $\mathsf{Explore}(s_1)$
		%\STATE {\bfseries set} $s_{h + 1} = \pi(s_h)$.
		%\ENDFOR
		\STATE {\bfseries return} $\pi$
	\end{algorithmic}
\end{algorithm}

\begin{algorithm}[tb]
	\caption{$\mathsf{Explore}(s)$}
	\label{alg:explore}
	\begin{algorithmic}[1]
		\FOR{$a \in \mathcal{A}$}
		\IF{$\phi(s, a)^{\top} C^{-1} \phi(s, a) \le 1$} \label{line:dsec}
		\STATE {\bfseries set} $\hat{Q}(s, a) = \phi(s, a)^{\top}C^{-1} Y$ \label{calc_Q_reg}
		\ELSE
		\STATE {\bfseries let} $s' = P(s, a)$ \label{start_recur}
		\STATE {\bfseries set} \[\hat{Q}(s, a) = 
		\begin{cases}
		r(s, a) & \text{if $s \in \states_H$}\\
		\mathsf{Explore}(s') + r(s, a) & \text{otherwise} 
		\end{cases}
		\] \label{calc_Q_recur}
		\STATE {\bfseries set} $C = C + \phi(s,a)\phi(s,a)^{\top}$
		\STATE {\bfseries set} $Y = Y + \phi(s, a) \hat{Q}(s, a)$ \label{add_data}
		\ENDIF
		\ENDFOR
		\STATE {\bfseries set} $\pi(s) = \mathrm{argmax}_{a \in \mathcal{A}} \hat{Q}(s, a)$.  \label{explore:set}
		\STATE {\bfseries return} 
		\[
		\begin{cases}
		r(s,\pi(s)) + \mathsf{Explore}(P(s,\pi(s))) & \text{if $s \in \states_H$}\\
		r(s,\pi(s)) & \text{otherwise}
		\end{cases}
		\]
		 \label{calc_V}
	\end{algorithmic}
\end{algorithm}

\subsection{The Analysis}\label{sec:analysis_linear}
In this section, we give the formal analysis of our algorithm.
Our goal is to show when $\gapmin \ge 4\delta(\sqrt{2d \log(16/\gapmin^2)} + 1)$, our algorithm learns the optimal policy $\pi^*$ using nearly linear number of trajectories.
\begin{theorem}	
	Suppose $\gapmin \geq 4\delta(\sqrt{2d \log(16/\gapmin^2)} + 1)$. 
	Algorithm \ref{alg:main} returns the optimal policy $\pi^*$ using at most $O(d \log(1/\gapmin))$ trajectories.
\end{theorem}
\begin{proof} 
	Recall that by Definition~\ref{def:approx_error} and Definition~\ref{def:linear_function}, there exists $\theta \in \mathbb{R}^d$ with $\|\theta\|_2 \le 1$ such that
	$
	|Q^*(s, a) - \theta^{\top}\phi(s, a)| \le \delta
	$
	for all state-action pairs $(s, a)$.
		
	Since the sample complexity of our algorithm equals the number of times we execute Line~\ref{start_recur} in $\mathsf{Explore}(s)$, following Lemma \ref{lemma:add_data}, the sample complexity of our algorithm is $O(d \log(1/\gapmin))$.
	
	To complete the proof, it is sufficient to prove the following induction hypothesis for all levels $h \in [H]$.
	\paragraph{Induction Hypothesis.} 
	\begin{enumerate}
		\item When Line~\ref{calc_Q_recur} is executed for any state $s\in \Scal_h$, $\hat{Q}(s, a) = Q^*(s, a)$.\label{ih:1}
		\item Each time Line~\ref{explore:set} in $\mathsf{Explore}(s)$ is executed for any state $s\in \Scal_h$, we have $\pi(s) = \pi^*(s)$, and the value returned by $\mathsf{Explore}(s)$ equals $V^*(s)$. \label{ih:2}
	\end{enumerate}
	
	For the above induction hypothesis, the base case $h = H$ is clearly true. 
	Now we assume the induction hypothesis holds for all levels $H, \ldots, h + 1$ and prove it holds for level $h$.
	\paragraph{Induction Hypothesis~\ref{ih:1}.} This follows from Induction Hypothesis~\ref{ih:2} for level $h+1$ and the Bellman equations. 
	\paragraph{Induction Hypothesis~\ref{ih:2}.} 
	By Induction Hypothesis~\ref{ih:1} and Definition~\ref{def:gap}, 
	we only need to show when Line~\ref{calc_Q_reg} is executed, we have $|\hat{Q}(s, a) - Q^*(s, a)| \le \gapmin / 2$, in which case we have $\pi(s) = \pi^*(s)$.
	 To verify this, note that
	\begin{align*}
		&|\phi(s, a)^{\top}C^{-1} Y - Q^*(s, a)| \\
		\le&|\phi(s, a)^{\top}C^{-1} Y -  \theta^{\top} \phi(s, a)| 
		+|Q^*(s, a) -  \theta^{\top} \phi(s, a)|.
	\end{align*}
	The second term is bounded by $\delta$.	
	For the first term, we write $\Phi$ to be a matrix whose $i$-th column is the $i$-th $\phi(s, a)$ vector in the summation.
	Recall that \[C = \left(\sum \phi(s, a)\phi(s, a)^{\top}\right) + \gapmin^2/16 \cdot I = \Phi \Phi^{\top} + \gapmin^2/16 \cdot I\] and \[Y = \sum \phi(s, a) Q^*(s, a)\] by Induction Hypothesis~\ref{ih:1}.
	Moreover, 
	\[Y = \sum \phi(s, a)(\phi(s, a)^{\top} \theta + b(s, a))\]
	where $|b(\cdot, \cdot)| \le \delta$.
	Thus, the first term can be upper bounded by
	\[
	\|\phi(s, a)^{\top}C^{-1} \Phi \|_1\cdot \delta + \left| \phi(s, a)^{\top}( C^{-1} \Phi \Phi^{\top} - I) \theta\right|.
	\]
	For the first term, by Lemma \ref{lemma:add_data} there are at most $2d \log(16/\gapmin^2)$ columns in $\Phi$. 
	When Line~\ref{calc_Q_reg} is executed, we must have $\phi(s, a)^{\top} C^{-1} \phi(s, a) \le 1$.
	Using Lemma~\ref{lem:ridge_reg} we have
	\begin{align*}
		&\|\phi(s, a)^{\top}C^{-1} \Phi \|_1\\
		 \le&  \sqrt{2d \log(16/\gapmin^2)}  \cdot \|\phi(s, a)^{\top}C^{-1} \Phi \|_2\\
		 = & \sqrt{2d \log(16/\gapmin^2)}  \cdot \sqrt{\phi(s, a)^{\top}C^{-1} \Phi \Phi^{\top} C^{-1}\phi(s, a)}\\
		\le & \sqrt{2d \log(16/\gapmin^2)}.
	\end{align*}

	For the second term, since $\|\theta\|_2 \le 1$ and $\phi(s, a)^{\top} C^{-1} \phi(s, a) \le 1$, by Cauchy-Schwarz and Lemma~\ref{lem:ridge_reg}, we have 
	\begin{align*}
	 &|\phi(s, a)^{\top}( C^{-1} \Phi^{\top} \Phi - I) \theta|\\
	  \le& \|\phi(s, a)^{\top}( C^{-1} \Phi^{\top} \Phi - I)\|_2\\
	  \le &   \gapmin/4.
	 \end{align*}
	  All together we get 	\begin{align*}
	|\phi(s, a)^{\top}C^{-1} Y - Q^*(s, a)| \le \gapmin / 2
	\end{align*} which completes the proof.
\end{proof}

\begin{lemma}
	\label{lemma:add_data}
	Line~\ref{add_data} is executed for at most $2d \log(16/\gapmin^2)$ times.
\end{lemma} 
\begin{proof}
	Suppose Line~\ref{add_data} has been executed for $T$ times, since $\|\phi(s, a)\|_2 \le 1$, 
	the trace of $ \phi(s, a)\phi(s, a)^{\top}$ is upper bounded by $\|\phi(s, a)\|_2^2 \le 1$.
	By additivity of trace, the trace of $C$ is upper bounded by \[T + d \cdot \gapmin^2/16\]
	since initially the trace of $C$ is $d \cdot \gapmin^2/16$.
	By AM-GM, \[\mathrm{det}(C) \le (T/d + \gapmin^2/16)^d.\]
	However, each time Line~\ref{add_data} is executed, by  matrix determinant lemma, $\mathrm{det}(C)$ will be increased by a factor of \[1 + \phi(s_h, a)^{\top} C^{-1} \phi(s_h, a) \ge 2.\]
	Moreover, initially $\mathrm{det}(C) = (\gapmin^2/16)^{d} $.
	Thus, \[2^T (\gapmin^2/16)^{d} \le (T/d + \gapmin^2/16)^d,\] which proves the lemma.
\end{proof}

\section{General Result}
\label{sec:general}
In this section, we consider the general case where $\Fcal$ is an arbitrary function class and provide a provably efficient algorithm which is a generalization of the algorithm in Section~\ref{sec:linear}.
Note that we make no assumptions on the action space $\Acal$. 
For simplicity, we assume that the reward is deterministic. We show how to remove this assumption in Section~\ref{app:reward}. We first define the Maximum Uncertainty Oracle which allows us to work with arbitrary action space.

\subsection{Maximum Uncertainty Oracle} \label{def:oracle}
As discussed in Section \ref{sec:linear-high}, it is useful to identify actions for which we can not accurately compute the optimal $Q$-value using the least-squares predictor. We formalize this intuition to arrive at the following oracle which finds the action with largest \say{uncertainty} for a given state $s$.
We note that similar oracles were also used in~\citep{du2019provably}.
\begin{definition}[$\mathsf{Oracle}(s,\delta, Y)$]
	Given a state $s \in \mathcal{S}$, $\delta \ge 0$ and a set of state-action pairs $Y \subseteq \mathcal{S} \times \mathcal{A}$, define
	\begin{align}
	(\hat a, \hat f_1, \hat f_2) &= \argmax_{a\in A, f_1,f_2 \in \Fcal} |f_{1}(s,a) - f_2(s,a)|^2 \label{cons:1}\\
	\text{s.t.}&\quad \frac{1}{|Y|}	\sum_{(s',a')\in Y}|f_{1}(s',a') - f_2(s',a')|^2\leq \delta^2. \label{cons:2}
	\end{align} The oracle returns $(\hat a, |\hat f_{1}(s,\hat a) - \hat f_2(s,\hat a)|^2)$.
\end{definition}

To motivate this oracle, suppose $f_2$ is the function that gives the best approximation of the optimal $Q$-function, i.e., the optimizer $f$ in Definition \ref{def:approx_error}.
In this scenario, we know $f_1$ predicts well on state-action pairs $(s',a')\in Y$ which is implied by the constraint.
Note that since we maximize over the entire function class $\Fcal$, $\hat a$ is the action with largest uncertainty. 
If $|\hat f_{1}(s,\hat a) - \hat f_2(s,\hat a)|^2$ is small, then we can predict well on state $s$ for all actions. Otherwise, it could be the case that we can not predict well on state $s$ for some action, so we need to explore and return the action with largest uncertainty. 
\begin{remark}\label{rmk:oracle}
	When $\Fcal$ is the class of linear functions, evaluating the oracle's response amounts to solving: \begin{align*}
	&(\hat a, \hat \theta_1, \hat \theta_2) = \argmax_{a\in A, \theta_1,\theta_2 \in \Fcal} |(\theta_{1} - \theta_2)^\top \phi(s,a)|^2 \\
	&\text{s.t.} \quad (\theta_{1} - \theta_2)^\top	\left(\frac{1}{|Y|} \sum_{(s',a')\in Y} \phi(s',a')\phi(s',a')^\top\right)(\theta_{1} - \theta_2) \leq \delta^2.
	\end{align*} %Then the oracle returns $(\hat a, |(\hat \theta_{1} - \hat \theta_2)^\top \phi(s,\hat a)|^2)$. 
	%As observed in \cite{du2019provably}, when the action space is finite, this amounts to solving the top eigenvalue problem which can be done efficiently.
	In this case, using the notation in the algorithm in Section~\ref{sec:linear}, it can be seen that the oracle returns the action $a \in \mathcal{A}$ which maximizes $\phi(s, a)^{\top}C^{-1} \phi(s, a)$.
\end{remark}

\subsection{Algorithm}
Similar to the algorithm for linear functions given in Section \ref{sec:linear}, the algorithm for general function class is divided into two parts: Algorithm \ref{alg:g-main} and a subroutine $\mathsf{Explore}(s)$. Intuitively, the subroutine $\mathsf{Explore(s)}$ should return $V^*(s)$, and upon the termination of $\mathsf{Explore}(s)$, we should have $\pi(s) = \pi^*(s)$. We will formally prove these in Section \ref{sec:Eluder-analysis}.

In our algorithm, we maintain a dataset to store the state-action pairs $(s,a)$ and their optimal $Q$-values $Q^*(s,a)$. In order to predict the optimal $Q$-value of an unseen state-action pair $(s,a)$, we find the best predictor on the dataset using least squares, and use it to predict on $(s,a)$. 

Similar to the algorithm in Section \ref{sec:linear}, the high level idea is that we use least squares to predict the optimal $Q$-value whenever possible, and otherwise we explore the environment. In Line \ref{line:check}, we check for a state $s$, whether the Maximum Uncertainty Oracle reports an uncertainty $r > |\rho/2 - \delta|$. As we will show in Section \ref{sec:Eluder-analysis}, such a choice guarantees that the value returned by $\mathsf{Explore}(s)$ always equals $V^*(s)$ and also, the number of times we explore, i.e., execute Line \ref{g-error}, is upper bounded by the Eluder dimension of function class $\Fcal$.

We remark that when applied to linear functions, using the notation in the algorithm in Section~\ref{sec:linear}, the subroutine $\mathsf{Explore(s)}$ keeps finding an action $a \in \mathcal{A}$ which maximizes $\phi(s, a)^{\top}C^{-1} \phi(s, a)$ (see Remark~\ref{rmk:oracle}) until $\phi(s, a)^{\top}C^{-1} \phi(s, a)$ is below a threshold for all actions $a \in \mathcal{A}$. Therefore, our algorithm is a generalization of the algorithm in Section~\ref{sec:linear}.
\iffalse
$Y$ as the set of state-action pairs where there was uncertainty in the choice of $f$. Therefore, by exploring the particular state-action pair, we can decrease the uncertainty about the choice of $f$. %Think of gap $\rho$ as the amount of error we are okay with propagating in our estimate of $f$. 
Our analysis will bound the size of $Y$ in terms of $\rho$ and $\delta$ (the approximation error).

\begin{align*}
\label{eq:choice-gap}
\rho \geq \max\bigg\{&2\left(2\sqrt{\frac{c\dim_E(F, \delta)-1}{c-1}} + 1\right) \delta,\\ &4\delta, 2\sqrt{c \dim_E(F, \delta)} \delta\bigg\}
\end{align*}
\fi
\subsection{Analysis}
In this section, we give the formal analysis of our algorithm. Our goal is to show that when $\rho \geq 6\sqrt{2}\delta\sqrt{\dim_E(\Fcal, \rho/4)}$, our algorithm learns the optimal policy $\pi^*$ using linear number of trajectories (in terms of Eluder dimension).
\label{sec:Eluder-analysis}

\begin{algorithm}[tb]
	\caption{Main Algorithm}
	\label{alg:g-main}
	\begin{algorithmic}[1]
		%\STATE {\bfseries Input:} precision $\delta$, horizon $H$, dimension $d$
		\STATE Initialize the current policy $\pi$ and $f$ arbitrarily.
		\STATE {\bfseries set} $Y  = \{\}$
		%\FOR{$h=H-1$ {\bfseries to} $1$}
		%\STATE Sample a trajectory $s_1, a_1, r_1, \ldots, s_h$ using $\pi$.
		\STATE {\bfseries invoke} $\mathsf{Explore}(s_1)$
		\STATE {\bfseries return} $\pi$
		%\ENDFOR
	\end{algorithmic}
\end{algorithm}

\begin{algorithm}[tb]
	\caption{$\mathsf{Explore}(s)$}
	\label{alg:g-explore}
	\begin{algorithmic}[1]
		%\STATE {\bfseries Input:} state $s$
		\STATE {\bfseries set} $(a, r) = \mathsf{Oracle}(s,2\delta, Y)$
		\WHILE{$r> |\frac{\rho}{2} - \delta|$} \label{line:check}
		\STATE {\bfseries set} \begin{equation*}
		Y = \begin{cases}
		Y \cup  \{(s, a, r(s, a))\} & s \in \Scal_H\\
		Y \cup  \{(s, a, \mathsf{Explore}(P(s,a)) + r(s, a))\} & \text{otherwise}
		\end{cases}
		\end{equation*}\label{g-error}
		\STATE {\bfseries set} $(a, r) = \mathsf{Oracle}(s,2\delta, Y)$
		\ENDWHILE
		\STATE {\bfseries set}  $f = \argmin_{f \in \Fcal} \sum_{(s_{i}, a_{i}, y_i) \in Y} | f(s_i, a_i) - y_i|^2$ \label{g-optimization}
		\STATE {\bfseries set}  $\pi(s) = \argmax_{a \in \mathcal{A}} f(s, a)$ \label{g-exploit}
		\STATE {\bfseries return} \begin{equation*}
		\begin{cases}
		r(s,\pi(s)) & s \in \Scal_H\\
		r(s,\pi(s)) + \mathsf{Explore}(P(s,\pi(s))) & \text{otherwise}
		\end{cases}
		\end{equation*} \label{g-calc_V}
	\end{algorithmic}
\end{algorithm}

\begin{theorem}	
	\label{thm:eluder-main}
Suppose\begin{equation}
		\label{eq:choice-gap}
		\rho \geq 6\sqrt{2}\delta\sqrt{\dim_E(\Fcal, \frac{\rho}{4})} .
	\end{equation} Then Algorithm \ref{alg:g-main} returns the optimal policy $\pi^*$ using at most $O(\dim_E(\Fcal,\rho/4))$ trajectories.
\end{theorem}
\begin{proof}
Firstly, using Lemma \ref{lemma:bound} with $c = 18$
we have \begin{equation}
\label{eq:y-bound}
	|Y|\leq 18\dim_E(\Fcal, \frac{\rho}{4}),
\end{equation}i.e. Line~\ref{g-error} is executed for at most $18 \dim_E(F, \rho/4)$ times and therefore the sample complexity of our algorithm is $O(\dim_E(\Fcal,\rho/4))$.

To complete the proof, it is sufficient to prove the following induction hypothesis for all levels $h \in [H]$.
\paragraph{Induction Hypothesis.} 
\begin{enumerate}
	\item For any state $s\in \Scal_h$, when Line~\ref{g-optimization} in $\mathsf{Explore}(s)$ is executed, we have $y_i = Q^*(s_i,a_i)$ for all $(s_i, a_i, y_i) \in Y$.\label{ihe:1}
	\item For any state $s\in \Scal_h$, when Line~\ref{g-exploit} in $\mathsf{Explore}(s)$ is executed, we have $\pi(s) = \pi^*(s)$, and the value returned by $\mathsf{Explore}(s)$ is $V^*(s)$. \label{ihe:2}
\end{enumerate}

	%Note that the base case $h = H-1$ is true. Let's assume the Induction Hypothesis holds for all levels $H-1, \ldots, h +1$ and prove it for $h$.
	For the above induction hypothesis, the base case $h = H$ is clearly true. Now we assume the induction hypothesis holds for all levels $H, \ldots,h+1$ and prove it holds for level $h$.
\paragraph{Induction Hypothesis~\ref{ihe:1}.} From Induction Hypothesis~\ref{ihe:2} for level $h+1$, it follows that value returned by $\mathsf{Explore}(P(s, a))$ is $V^*(P(s, a))$ for all $a \in \Acal$. Then, Induction Hypothesis~\ref{ihe:1} follows from the Bellman equations.
\paragraph{Induction Hypothesis~\ref{ihe:2}.} It suffices to show that for any state $s\in \Scal_h$, when Line~\ref{g-exploit} in $\mathsf{Explore}(s)$ is executed, for all actions $a \in \Acal$ \begin{equation}
\label{eq:3}
| f(s, a) - Q^*(s,a)| \leq \frac{\rho}{2}.
\end{equation}  

First, there exists $f^* \in \Fcal$ such that for all $(s_{i}, a_{i}, y_i) \in Y$, \begin{equation}
\label{eq:1}
| f^*(s_i, a_i) - Q^*(s_i,a_i)| \leq \delta.
\end{equation} From Induction Hypothesis~\ref{ihe:1}, for all $(s_i, a_i, y_i) \in Y$ \begin{equation}
\label{eq:2}
y_i = Q^*(s_i,a_i).
\end{equation} From Equation \eqref{eq:1} and \eqref{eq:2}, it follows that\begin{equation}
\label{eq:4}
\sum_{(s_{i}, a_{i}, y_i) \in Y} | f^*(s_i, a_i) - y_i|^2 \leq |Y| \delta^2.
\end{equation}  Since, we execute Line \ref{g-optimization} and $f^* \in \Fcal$, from Equation \eqref{eq:4}, it follows that  \begin{equation}
\label{eq:5}
\sum_{(s_{i}, a_{i}, y_i) \in Y} | f(s_i, a_i) - y_i|^2 \leq |Y| \delta^2.
\end{equation}

We split the analysis into two cases: \begin{enumerate}
	\item we consider actions for which we execute Line~\ref{g-error} and\label{case:1}
	\item we consider rest of the actions.\label{case:2}
\end{enumerate}

\paragraph{Case \ref{case:1}:} We now prove Equation~\eqref{eq:3} for all actions $a$ for which we execute Line \ref{g-error}. Using Equation \eqref{eq:y-bound}, \eqref{eq:2} and \eqref{eq:5}, we get that for actions $a$ for which we executed Line~\ref{g-error} (since then we added it to $Y$) \begin{equation}
	| f(s, a) - Q^*(s,a)| \leq \sqrt{18 \dim_E(F, \frac{\rho}{4})} \delta \leq \frac{\rho}{2}
\end{equation} where the last step follows from our assumption on $\rho$ (Equation \eqref{eq:choice-gap}).

\paragraph{Case \ref{case:2}:} We now prove this for rest of the actions $a$. From Equation \eqref{eq:1}, \eqref{eq:2}, \eqref{eq:5} and triangle inequality for the $\ell_2$ norm, we get \begin{equation}
\sum_{(s_{i}, a_{i}, y_i) \in Y} | f^*(s_i, a_i) - f(s_i, a_i)|^2 \leq 4   |Y| \delta^2.
\end{equation} Also, since we did not add this action to $Y$, by the definition of the oracle (Definition \ref{def:oracle}), we  get\begin{equation}
| f^*(s, a) - f(s, a)| \leq \frac{\rho}{2} - \delta.
\end{equation} 
Therefore, \begin{equation}
| Q^*(s, a) - f(s, a)| \leq \frac{\rho}{2}
\end{equation} which completes the proof.
\end{proof}

For the sample complexity, we use the following lemma.
\begin{lemma} 
	\label{lemma:bound}
	For any constant $c>1$, suppose \begin{equation}
	\rho \geq 4\delta\sqrt{\frac{c\dim_E(F, \frac{\rho}{4})-1}{c-1}} + 2\delta,
	\end{equation}
	then we have
	\begin{equation}
	|Y|\leq c\dim_E(F, \frac{\rho}{4}).
	\end{equation}
\end{lemma}
The proof relies on definition of the Eluder dimension and the Maximum Uncertainty Oracle.
See the supplementary material for the formal proof.

\section{Conclusion}
\label{sec:discussion}
In this paper, we propose a novel provably efficient recursion-based algorithm for agnostic $Q$-learning with general function approximation with bounded Eluder dimension in deterministic systems.
We obtain a sharp characterization on the relation between the approximation error and the optimality gap, and also a tight sample complexity.
We thus settle the open problem raised by \citet{wen2013efficient}.

\bibliography{simonduref,references}

\begin{thebibliography}{34}
\providecommand{\natexlab}[1]{#1}
\providecommand{\url}[1]{\texttt{#1}}
\expandafter\ifx\csname urlstyle\endcsname\relax
  \providecommand{\doi}[1]{doi: #1}\else
  \providecommand{\doi}{doi: \begingroup \urlstyle{rm}\Url}\fi

\bibitem[Watkins and Dayan(1992)]{watkins1992q}
Christopher~JCH Watkins and Peter Dayan.
\newblock Q-learning.
\newblock \emph{Machine learning}, 8\penalty0 (3-4):\penalty0 279--292, 1992.

\bibitem[Strehl et~al.(2006)Strehl, Li, Wiewiora, Langford, and
  Littman]{strehl2006pac}
Alexander~L Strehl, Lihong Li, Eric Wiewiora, John Langford, and Michael~L
  Littman.
\newblock {PAC} model-free reinforcement learning.
\newblock In \emph{Proceedings of the 23rd international conference on Machine
  learning}, pages 881--888. ACM, 2006.

\bibitem[Jin et~al.(2018)Jin, Allen-Zhu, Bubeck, and Jordan]{jin2018q}
Chi Jin, Zeyuan Allen-Zhu, Sebastien Bubeck, and Michael~I Jordan.
\newblock Is {Q}-learning provably efficient?
\newblock In \emph{Advances in Neural Information Processing Systems}, pages
  4863--4873, 2018.

\bibitem[Wen and Van~Roy(2013)]{wen2013efficient}
Zheng Wen and Benjamin Van~Roy.
\newblock Efficient exploration and value function generalization in
  deterministic systems.
\newblock In \emph{Advances in Neural Information Processing Systems}, pages
  3021--3029, 2013.

\bibitem[Du et~al.(2019)Du, Luo, Wang, and Zhang]{du2019provably}
Simon~S Du, Yuping Luo, Ruosong Wang, and Hanrui Zhang.
\newblock Provably efficient {$Q$}-learning with function approximation via
  distribution shift error checking oracle.
\newblock In \emph{Advances in Neural Information Processing Systems}, pages
  8058--8068, 2019.

\bibitem[Du et~al.(2020)Du, Kakade, Wang, and Yang]{du2019good}
Simon~S Du, Sham~M Kakade, Ruosong Wang, and Lin~F Yang.
\newblock Is a good representation sufficient for sample efficient
  reinforcement learning?
\newblock In \emph{International Conference on Learning Representations}, 2020.

\bibitem[Lattimore and Szepesvari(2019)]{lattimore2019learning}
Tor Lattimore and Csaba Szepesvari.
\newblock Learning with good feature representations in bandits and in {RL}
  with a generative model.
\newblock \emph{arXiv preprint arXiv:1911.07676}, 2019.

\bibitem[Van~Roy and Dong(2019)]{van2019comments}
Benjamin Van~Roy and Shi Dong.
\newblock Comments on the {Du-Kakade-Wang-Yang} lower bounds.
\newblock \emph{arXiv preprint arXiv:1911.07910}, 2019.

\bibitem[Neu and Olkhovskaya(2020)]{neu2020efficient}
Gergely Neu and Julia Olkhovskaya.
\newblock Efficient and robust algorithms for adversarial linear contextual
  bandits.
\newblock \emph{arXiv preprint arXiv:2002.00287}, 2020.

\bibitem[Mannor and Tsitsiklis(2004)]{mannor2004sample}
Shie Mannor and John~N Tsitsiklis.
\newblock The sample complexity of exploration in the multi-armed bandit
  problem.
\newblock \emph{Journal of Machine Learning Research}, 5\penalty0
  (Jun):\penalty0 623--648, 2004.

\bibitem[Melo and Ribeiro(2007)]{melo2007q}
Francisco~S Melo and M~Isabel Ribeiro.
\newblock Q-learning with linear function approximation.
\newblock In \emph{International Conference on Computational Learning Theory},
  pages 308--322, 2007.

\bibitem[Zou et~al.(2019)Zou, Xu, and Liang]{zou2019finite}
Shaofeng Zou, Tengyu Xu, and Yingbin Liang.
\newblock Finite-sample analysis for {SARSA} with linear function
  approximation.
\newblock In \emph{Advances in Neural Information Processing Systems}, pages
  8665--8675, 2019.

\bibitem[Lattimore and Hutter(2012)]{lattimore2012pac}
Tor Lattimore and Marcus Hutter.
\newblock {PAC} bounds for discounted mdps.
\newblock In \emph{International Conference on Algorithmic Learning Theory},
  pages 320--334. Springer, 2012.

\bibitem[Azar et~al.(2013)Azar, Munos, and Kappen]{azar2013minimax}
Mohammad~Gheshlaghi Azar, R{\'e}mi Munos, and Hilbert~J Kappen.
\newblock Minimax {PAC} bounds on the sample complexity of reinforcement
  learning with a generative model.
\newblock \emph{Machine learning}, 91\penalty0 (3):\penalty0 325--349, 2013.

\bibitem[Sidford et~al.(2018{\natexlab{a}})Sidford, Wang, Wu, Yang, and
  Ye]{sidford2018near}
Aaron Sidford, Mengdi Wang, Xian Wu, Lin Yang, and Yinyu Ye.
\newblock Near-optimal time and sample complexities for solving markov decision
  processes with a generative model.
\newblock In \emph{Advances in Neural Information Processing Systems}, pages
  5186--5196, 2018{\natexlab{a}}.

\bibitem[Sidford et~al.(2018{\natexlab{b}})Sidford, Wang, Wu, and
  Ye]{sidford2018variance}
Aaron Sidford, Mengdi Wang, Xian Wu, and Yinyu Ye.
\newblock Variance reduced value iteration and faster algorithms for solving
  markov decision processes.
\newblock In \emph{Proceedings of the Twenty-Ninth Annual ACM-SIAM Symposium on
  Discrete Algorithms}, pages 770--787. SIAM, 2018{\natexlab{b}}.

\bibitem[Agarwal et~al.(2019)Agarwal, Kakade, and Yang]{agarwal2019optimality}
Alekh Agarwal, Sham Kakade, and Lin~F Yang.
\newblock On the optimality of sparse model-based planning for markov decision
  processes.
\newblock \emph{arXiv preprint arXiv:1906.03804}, 2019.

\bibitem[Jaksch et~al.(2010)Jaksch, Ortner, and Auer]{jaksch2010near}
Thomas Jaksch, Ronald Ortner, and Peter Auer.
\newblock Near-optimal regret bounds for reinforcement learning.
\newblock \emph{Journal of Machine Learning Research}, 11\penalty0
  (Apr):\penalty0 1563--1600, 2010.

\bibitem[Agrawal and Jia(2017)]{agrawal2017posterior}
Shipra Agrawal and Randy Jia.
\newblock Optimistic posterior sampling for reinforcement learning: worst-case
  regret bounds.
\newblock In \emph{Advances in Neural Information Processing Systems}, pages
  1184--1194, 2017.

\bibitem[Azar et~al.(2017)Azar, Osband, and Munos]{azar2017minimax}
Mohammad~Gheshlaghi Azar, Ian Osband, and R{\'e}mi Munos.
\newblock Minimax regret bounds for reinforcement learning.
\newblock In \emph{Proceedings of the 34th International Conference on Machine
  Learning}, pages 263--272, 2017.

\bibitem[Kakade et~al.(2018)Kakade, Wang, and Yang]{kakade2018variance}
Sham Kakade, Mengdi Wang, and Lin~F Yang.
\newblock Variance reduction methods for sublinear reinforcement learning.
\newblock \emph{arXiv preprint arXiv:1802.09184}, 2018.

\bibitem[{Azizzadenesheli} et~al.(2018){Azizzadenesheli}, {Brunskill}, and
  {Anandkumar}]{azizzadenesheli2018efficient}
K.~{Azizzadenesheli}, E.~{Brunskill}, and A.~{Anandkumar}.
\newblock Efficient exploration through bayesian deep {Q}-networks.
\newblock In \emph{2018 Information Theory and Applications Workshop (ITA)},
  pages 1--9, 2018.

\bibitem[Fortunato et~al.(2018)Fortunato, Azar, Piot, Menick, Hessel, Osband,
  Graves, Mnih, Munos, Hassabis, Pietquin, Blundell, and
  Legg]{fortunato2018noisy}
Meire Fortunato, Mohammad~Gheshlaghi Azar, Bilal Piot, Jacob Menick, Matteo
  Hessel, Ian Osband, Alex Graves, Volodymyr Mnih, Remi Munos, Demis Hassabis,
  Olivier Pietquin, Charles Blundell, and Shane Legg.
\newblock Noisy networks for exploration.
\newblock In \emph{International Conference on Learning Representations}, 2018.

\bibitem[Lipton et~al.(2018)Lipton, Li, Gao, Li, Ahmed, and
  Deng]{lipton2018bbq}
Zachary~Chase Lipton, Xiujun Li, Jianfeng Gao, Lihong Li, Faisal Ahmed, and
  Li~Deng.
\newblock {BBQ}-networks: Efficient exploration in deep reinforcement learning
  for task-oriented dialogue systems.
\newblock In \emph{AAAI}, 2018.

\bibitem[Osband et~al.(2016)Osband, Van~Roy, and Wen]{osband2016generalization}
Ian Osband, Benjamin Van~Roy, and Zheng Wen.
\newblock Generalization and exploration via randomized value functions.
\newblock In \emph{Proceedings of the 33rd International Conference on
  International Conference on Machine Learning}, pages 2377--2386, 2016.

\bibitem[Pazis and Parr(2013)]{pazis2013pac}
Jason Pazis and Ronald Parr.
\newblock {PAC} optimal exploration in continuous space markov decision
  processes.
\newblock In \emph{Proceedings of the Twenty-Seventh AAAI Conference on
  Artificial Intelligence}, pages 774--781, 2013.

\bibitem[Li et~al.(2011)Li, Littman, Walsh, and Strehl]{li2011knows}
Lihong Li, Michael~L Littman, Thomas~J Walsh, and Alexander~L Strehl.
\newblock Knows what it knows: a framework for self-aware learning.
\newblock \emph{Machine learning}, 82\penalty0 (3):\penalty0 399--443, 2011.

\bibitem[Yang and Wang(2019{\natexlab{a}})]{yang2019sample}
Lin~F. Yang and Mengdi Wang.
\newblock Sample-optimal parametric q-learning using linearly additive
  features.
\newblock In \emph{International Conference on Machine Learning}, pages
  6995--7004, 2019{\natexlab{a}}.

\bibitem[Yang and Wang(2019{\natexlab{b}})]{yang2019reinforcement}
Lin~F. Yang and Mengdi Wang.
\newblock Reinforcement leaning in feature space: Matrix bandit, kernels, and
  regret bound.
\newblock \emph{arXiv preprint arXiv:1905.10389}, 2019{\natexlab{b}}.

\bibitem[Jin et~al.(2019)Jin, Yang, Wang, and Jordan]{jin2019provably}
Chi Jin, Zhuoran Yang, Zhaoran Wang, and Michael~I Jordan.
\newblock Provably efficient reinforcement learning with linear function
  approximation.
\newblock \emph{arXiv preprint arXiv:1907.05388}, 2019.

\bibitem[Wang et~al.(2019)Wang, Wang, Du, and Krishnamurthy]{wang2019optimism}
Yining Wang, Ruosong Wang, Simon~S Du, and Akshay Krishnamurthy.
\newblock Optimism in reinforcement learning with generalized linear function
  approximation.
\newblock \emph{arXiv preprint arXiv:1912.04136}, 2019.

\bibitem[Simchowitz and Jamieson(2019)]{simchowitz2019non}
Max Simchowitz and Kevin~G Jamieson.
\newblock Non-asymptotic gap-dependent regret bounds for tabular mdps.
\newblock In \emph{Advances in Neural Information Processing Systems}, pages
  1151--1160, 2019.

\bibitem[Russo and Van~Roy(2013)]{russo2013eluder}
Daniel Russo and Benjamin Van~Roy.
\newblock Eluder dimension and the sample complexity of optimistic exploration.
\newblock In \emph{Advances in Neural Information Processing Systems}, pages
  2256--2264, 2013.

\bibitem[Li et~al.(2010)Li, Chu, Langford, and Schapire]{li2010contextual}
Lihong Li, Wei Chu, John Langford, and Robert~E Schapire.
\newblock A contextual-bandit approach to personalized news article
  recommendation.
\newblock In \emph{Proceedings of the 19th international conference on World
  wide web}, pages 661--670, 2010.

\end{thebibliography}
\bibliographystyle{unsrtnat}
\newpage
\appendix
\section{Extension to stochastic rewards}
\label{app:reward}

\iffalse
\paragraph{Generalize to stochastic reward.} We modify the algorithm by repeat Line~\ref{calc_V} for $O(\log(d) / \delta^2)$ times to reduce the variance. The $O(\log d)$ is due to the fact that we need a union bound over all visited states.
For the analysis, change Induction Hypothesis~\ref{ih:2} to ``the value returned by $\mathsf{Explore}(s_h)$ is in $[V(s_h) - \delta, V(s_h) + \delta]$''.
Similar change for Induction Hypothesis~\ref{ih:3}. 
When proving Induction Hypothesis~\ref{ih:2}, we need to increase $|b(s, a)|$ by an additive term of $\delta$, which plays essentially no role in the proof.
\fi

In this section, we extend our algorithm and analysis to stochastic rewards, i.e., reward $r(s,a) \sim R(s,a)$ is a random variable with expectation $\bar r(s,a)$ and $r(s,a) \in [0,1]$. 

\subsection{Algorithm}
We modify $\mathsf{Explore}(s)$ such that whenever previously we used $r(s,a)$, we use the empirical mean $\hat r(s,a)$ of $n$ samples from $R(s,a)$ to get a good estimate of the expected reward $\bar r(s,a)$. For our algorithm, we set \begin{equation}
n = \frac{H^2}{2\delta_r^2} \log\frac{18 \dim_E(\Fcal, \rho/4) H}{p},
\end{equation}
where $\delta_r$ is a parameter to be chosen and $p$ is the failure probability of the algorithm.

\begin{algorithm}[t]
	\caption{Main Algorithm}
	\label{alg:u-main}
	\begin{algorithmic}[1]
		%\STATE {\bfseries Input:} precision $\delta$, horizon $H$, dimension $d$
		\STATE Initialize the current policy $\pi$ and $f$ arbitrarily
		\STATE {\bfseries set} $Y  = \{\}$
		\STATE {\bfseries invoke} $\mathsf{Explore}(s_1)$
	\end{algorithmic}
\end{algorithm}

\begin{algorithm}[t]
	\caption{$\mathsf{Explore}(s)$}
	\label{alg:u-explore}
	\begin{algorithmic}[1]
		%\STATE {\bfseries Input:} state $s$
		\STATE {\bfseries set} $(a, r) = \mathsf{Oracle}(s,2(\delta + \delta_r), Y)$
		\WHILE{$r> |\frac{\rho}{2} - \delta|$} \label{line:u-check}
		\STATE {\bfseries set} $\hat r(s, a)$ to be the empirical mean of $n= \frac{H^2}{2\delta_r^2} \log\frac{18 \dim_E(\Fcal, \rho/4) H}{p}
$ samples from $R(s,a)$		
		\STATE {\bfseries set} \begin{equation*}
			Y = \begin{cases}
			Y \cup  \{(s, a, \hat r(s, a))\} & s \in \Scal_H\\
			Y \cup  \{(s, a, \mathsf{Explore}(P(s,a)) + \hat r(s, a))\} & \text{otherwise}
			\end{cases}
		\end{equation*} \label{u-error}
		\STATE {\bfseries set} $(a, r) = \mathsf{Oracle}(s,2(\delta + \delta_r), Y)$
		\ENDWHILE
		\STATE {\bfseries set}  $f = \argmin_{f \in \Fcal} \sum_{(s_{i}, a_{i}, y_i) \in Y} | f(s_i, a_i) - y_i|^2$ \label{u-optimization}
		\STATE {\bfseries set}  $\pi(s) = \argmax_{a \in \mathcal{A}} f(s, a)$  \label{u-exploit}
		\STATE {\bfseries return} \begin{equation*}
			\begin{cases}
			\hat r(s,\pi(s)) & s \in \Scal_H\\
			\hat r(s,\pi(s)) + \mathsf{Explore}(P(s,\pi(s))) & \text{otherwise}
			\end{cases}
		\end{equation*} \label{u-calc_V}
	\end{algorithmic}
\end{algorithm}
\iffalse
\begin{align*}
\label{eq:choice-gap}
\rho \geq \max\bigg\{&2\left(2\sqrt{\frac{c\dim_E(F, \delta)-1}{c-1}} + 1\right) \delta,\\ &4\delta, 2\sqrt{c \dim_E(F, \delta)} \delta\bigg\}
\end{align*}
\fi
\subsection{Analysis}
\begin{theorem}	
	\label{thm:random}
	Suppose \begin{equation}
	\label{eq:u-choice-gap}
	\rho \geq 6\sqrt{2}(\delta + \delta_r)\sqrt{\dim_E(\Fcal, \rho/4)} + 2\delta_r.
	\end{equation} 
	Algorithm \ref{alg:u-main} returns the optimal policy $\pi^*$ with probability $1-p$.
\end{theorem}
\begin{remark}
Note that by setting \begin{equation}
	\delta_r = \frac{\rho}{24\sqrt{2}\dim_E(\Fcal, \rho/4)}\quad \text{and} \quad \rho \geq 12 \sqrt{2} \delta \sqrt{\dim_E(\Fcal, \rho/4)},
\end{equation} 
Theorem~\ref{thm:random} implies that 
Algorithm \ref{alg:u-main} returns the optimal policy $\pi^*$ with probability $1-p$ using at most \[\frac{\poly(\dim_E(\Fcal, \rho/4), H)}{\rho^2} \log(1/p)\] trajectories.
\end{remark}
Now we formally prove Theorem~\ref{thm:random}.
\begin{proof}[Proof of Theorem \ref{thm:random}]
	Firstly, for $c=18$, following Lemma \ref{lemma:u-bound}, we have \begin{equation}
	\label{eq:u-y-bound}
	|Y|\leq 18\dim_E(\Fcal, \rho/4),
	\end{equation}
	i.e. Line~\ref{u-error} is executed for at most $18 \dim_E(F, \rho/4)$ times.\\
	\\
	To complete the proof, it is sufficient to prove the following induction hypothesis for all levels $h \in [H]$.
	\paragraph{Induction Hypothesis.} 
	\begin{enumerate}
		\item For any state $s\in \Scal_h$, when Line~\ref{u-optimization} in $\mathsf{Explore}(s)$ is executed, we have \[y_i \in \left[Q^*(s_i,a_i) - \frac{H-h+1}{H}\delta_r, Q^*(s_i,a_i) + \frac{H-h+1}{H}\delta_r\right] \] for all $(s_i, a_i, y_i) \in Y$.\label{ihu:1}
		\item For any state $s\in \Scal_h$, when Line~\ref{u-exploit} in $\mathsf{Explore}(s)$ is executed, we have $\pi(s) = \pi^*(s)$, and the value returned by $\mathsf{Explore}(s)$ is in 
		\[
		\left[V^*(s) - \frac{H-h+1}{H}\delta_r, V^*(s) + \frac{H-h+1}{H}\delta_r\right].
		\]
		 \label{ihu:2}
	\end{enumerate}
	
	Note that the base case $h = H$ is true by Lemma \ref{lemma:concentration} and union bound. 
	Now we assume the induction hypothesis holds for all levels $H, \ldots, h + 1$ and prove it holds for level $h$.

	\paragraph{Induction Hypothesis~\ref{ihu:1}.} From Induction Hypothesis~\ref{ihu:2} for level $h+1$, it follows that value returned by $\mathsf{Explore}(P(s, a))$ is in \[
	\left[V^*(P(s, a)) - \frac{H-h}{H}\delta_r, V^*(P(s,a)) + \frac{H-h}{H}\delta_r\right]
	\] for all $a \in \Acal$. Then, Induction Hypothesis~\ref{ihu:1} follows from Lemma \ref{lemma:concentration} and union bound.
	\paragraph{Induction Hypothesis~\ref{ihu:2}.} It suffices to show that for any state $s\in \Scal_h$, when Line~\ref{u-exploit} in $\mathsf{Explore}(s)$ is executed, then for all actions $a \in \Acal$ \begin{equation}
	\label{eq:3u}
	| f(s, a) - Q^*(s,a)| \leq \frac{\rho}{2}.
	\end{equation}  
	
	Similar to proof of Theorem \ref{thm:eluder-main}, we get
	\iffalse
	First, following Assumption \ref{assume:agnostic}, there exists $f^* \in \Fcal$ such that for all $(s_{i}, a_{i}, y_i) \in Y$, \begin{equation}
	\label{eq:1u}
	| f^*(s_i, a_i) - Q^*(s_i,a_i)| \leq \delta
	\end{equation} From Induction Hypothesis~\ref{ihu:1}, for all $(s_i, a_i, y_i) \in Y$ \begin{equation}
	\label{eq:2}
	|y_i - Q^*(s_i,a_i)| \leq \delta_r
	\end{equation} From Equation \eqref{eq:1} and \eqref{eq:2}, it follows that\begin{equation}
	\label{eq:4}
	\sum_{(s_{i}, a_{i}, y_i) \in Y} | f^*(s_i, a_i) - y_i|^2 \leq |Y| (\delta + \delta_r)^2
	\end{equation}  Since, we execute Line \ref{g-optimization} and $f^* \in \Fcal$, from Equation \ref{eq:4}, it follows that  
	\fi
	\begin{equation}
	\label{eq:5u}
	\sum_{(s_{i}, a_{i}, y_i) \in Y} | f(s_i, a_i) - y_i|^2 \leq |Y| (\delta + \delta_r)^2.
	\end{equation}
	
	We split the analysis in two cases: \begin{enumerate}
		\item we consider actions for which we execute Line~\ref{u-error} and\label{case:1u}
		\item we consider rest of the actions.\label{case:2u}
	\end{enumerate}
	
	\paragraph{Case \ref{case:1u}:} We now prove Equation \eqref{eq:3u} for all actions $a$ for which we execute Line \ref{u-error}. Similar to proof of Theorem \ref{thm:eluder-main}, we get \begin{align}
	| f(s, a) - Q^*(s,a)| &\leq \sqrt{18 \dim_E(\Fcal, \rho/4)} (\delta + \delta_r) + \delta_r\notag \\
	 &\leq \frac{\rho}{2}.
	\end{align}
	
	\paragraph{Case \ref{case:2u}:} We now prove this for rest of the actions $a$. Similar to proof of Theorem \ref{thm:eluder-main}, we get \begin{equation}
	\sum_{(s_{i}, a_{i}, y_i) \in Y} | f^*(s_i, a_i) - f(s_i, a_i)|^2 \leq 4   |Y| (\delta + \delta_r)^2.
	\end{equation} Also, since we did not add this action to $Y$, by the definition of the oracle (Definition \ref{def:oracle}),
	\iffalse, we  get\begin{equation}
	| f^*(s, a) - f(s, a)| \leq \frac{\rho}{2} - \delta
	\end{equation} Because of 
	\fi
	we get \begin{equation}
	| Q^*(s, a) - f(s, a)| \leq \frac{\rho}{2},
	\end{equation} which completes the proof.
\end{proof}

\begin{lemma} 
	\label{lemma:u-bound}
	For any constant $c>1$, if \begin{equation}
	\rho \geq 4(\delta + \delta_r)\sqrt{\frac{c\dim_E(\Fcal, \rho/4)-1}{c-1}} + 2\delta,
	\end{equation} then
	\begin{equation}
	|Y|\leq c\dim_E(\Fcal, \rho/4).
	\end{equation}
\end{lemma}
\begin{proof}
	Let $Y = \{(s_1,a_1,y_1),\ldots, (s_n,a_n,y_n)\}$.
	\iffalse When we add $(s_j, a_j, y_j)$ to $Y$ at Line \ref{u-error},  we get
	
	\begin{enumerate}
		\item The condition at Line \ref{line:check} must be True i.e. from Equation \eqref{cons:1}, there exists $f_1,f_2 \in F$ such that $|f_{1}(s_j, a_j) - f_{2}(s_j, a_j)| > \frac{\rho}{2} - \delta$.
		\item Observe that for any subsequence $B\subset \{(s_1, a_1), \ldots, (s_{j-1}, a_{j-1})\}$ where $(s_j,a_j)$ is $(\frac{\rho}{2} - \delta)$-dependent on $B$ (Definition \ref{def:ind}), \begin{equation}
		\sum_{(s,a)\in B} |f_{1}(s,a) - f_{2}(s,a)|^2 \geq (\frac{\rho}{2} - \delta)^2
		\end{equation}
		\item Therefore, if there are $K$ disjoint subsequences in $\{(s_1, a_1), \ldots, (s_{j-1}, a_{j-1})\}$ such that $(s_j,a_j)$ is $(\frac{\rho}{2} - \delta)$-dependent on all of them, then \begin{equation}
		\label{eq:lower-bound}
		\sum_{i=1}^{j-1}  |f_{1}(s_i,a_i) - f_{2}(s_i,a_i)|^2 \geq K (\frac{\rho}{2} - \delta)^2
		\end{equation}
		\item However, using Equation \ref{cons:2}, we have that\begin{equation}
		\label{eq:upper-bound}
		\sum_{i=1}^{j-1}  |f_{1}(s_i,a_i) - f_{2}(s_i,a_i)|^2 \leq (j-1) (2(\delta + \delta_r))^2
		\end{equation}
	\end{enumerate}
	
	Therefore, \fi 
	Similar to proof of Lemma \ref{lemma:bound}, we can upper bound for any state-action pair $(s_j,a_j)\in \{(s_1,a_1),\ldots, (s_n,a_n)\}$, the number of disjoint subsequences $K$ in $\{(s_1, a_1), \ldots, (s_{j-1}, a_{j-1})\}$ that $(s_j,a_j)$ is $(\frac{\rho}{2} - \delta)$-dependent on, i.e. \[
	K \leq  \frac{(j-1) (2(\delta + \delta_r))^2}{(\frac{\rho}{2} - \delta)^2}.
	\]
	\iffalse Moreover, it follows from the proof of \cite{russo2013eluder}[Proposition 3] that \fi
	 Also, for any sequence of state-action pairs say $\{(s_1,a_1),\ldots, (s_n,a_n)\}$, there exists a $(s_j,a_j)$ which is $(\frac{\rho}{2} - \delta)$-dependent on at least  $\frac{n}{\dim_E(F, \frac{\rho}{2} - \delta)} - 1$ disjoint subsequences in $\{(s_1,a_1),\ldots, (s_{j-1},a_{j-1})\}$. Therefore, \begin{equation}
	\frac{n}{\dim_E(F, \frac{\rho}{2} - \delta)} - 1 \leq K \leq \frac{(j-1) (2(\delta + \delta_r))^2}{(\frac{\rho}{2} - \delta)^2}.
	\end{equation} 
	\iffalse and as $j$ is atleast as large as $n$\begin{equation}
	n \leq \dim_E(F, \frac{\rho}{2} - \delta) \left(\frac{(n-1) (2(\delta + \delta_r))^2}{(\frac{\rho}{2} - \delta)^2} + 1\right)
	\end{equation} As $\rho>4\delta$, we get \begin{equation}
	n \leq \dim_E(\Fcal, \rho/4) \left(\frac{(n-1) (2(\delta + \delta_r))^2}{(\frac{\rho}{2} - \delta)^2} + 1\right)
	\end{equation} which follows from definition of Eluder dimension since $a<b$ implies $\dim_E(F,a)\geq \dim_E(F,b)$. 
	\fi  
	That is, for any $\rho$ and $c>1$ such that \begin{equation}
	\rho \geq 2\left(2(\delta + \delta_r)\sqrt{\frac{c\dim_E(\Fcal, \rho/4)-1}{c-1}} + \delta\right),
	\end{equation} we get
	\begin{equation}
	n\leq c\dim_E(\Fcal, \rho/4).
	\end{equation}
\end{proof}

A simple concentration bound gives the following lemma:
\begin{lemma}
	\label{lemma:concentration}
	For any fixed state $s$ and action $a$, consider $n \geq \frac{H^2}{2\delta_r^2} \log\frac{1}{p}$ random independent samples $\{r_i(s,a)\}_{i=1}^n$ of random variable $R(s,a)$ with expectation $\bar r(s,a)$ and $r_i(s,a) \in [0,1]$. Then, \[
	\left|\frac{1}{n} \sum_{i=1}^n r_i(s,a) - \bar r(s,a)\right| \leq \frac{\delta_r}{H}
	\] with probability at least $1- p$.
\end{lemma}
% !TEX root = main.tex
\section{Proofs for Section \ref{sec:linear}}

\begin{lemma}\label{lem:ridge_reg}
	For any positive semi-definite $M \in \mathbb{R}^{d \times d}$, $\alpha > 0$ and $x \in \mathbb{R}^d$ such that $x^{\top} (M +\alpha \cdot I)^{- 1} x \le 1$, we have
	\begin{itemize}
		\item $\| (M(M + \alpha \cdot I)^{-1} - I) x \|_2 \le \alpha$;
		\item $x^{\top} (M + \alpha \cdot I)^{-1} M  (M + \alpha \cdot I)^{-1} x \le 1$.
	\end{itemize}
\end{lemma}
\begin{proof}
	We use $M = U^T \Lambda U$ to denote the spectral decomposition of $M$, where $\Lambda$ is a diagonal matrix with non-negative entries.
	We use $\Lambda_i$ to denote the $i$-th diagonal entry of $\Lambda$ and let $y = Ux$.
	By the assumption, it holds that
	\[
	\sum_{i = 1}^d \frac{y_i^2}{\Lambda_i + \alpha} \le 1.
	\]
	Clearly,
	\begin{align*}
	&\| (M(M + \alpha \cdot I)^{-1} - I) x \|_2^2 \\
	=&\sum_{i = 1}^d y_i^2 \cdot \left( \frac{\Lambda_i}{\Lambda_i + \alpha} - 1\right)^2 = \sum_{i = 1}^d y_i^2 \cdot \left( \frac{\alpha}{\Lambda_i + \alpha} \right)^2 \le \alpha
	\end{align*}
	and
	\begin{align*}
	&x^{\top} (M + \alpha \cdot I)^{-1} M  (M + \alpha \cdot I)^{-1} x\\
	=&\sum_{i = 1}^d y_i^2 \cdot \frac{\Lambda_i}{(\Lambda_i + \alpha \cdot I)^2} \le 1.
	\end{align*}
\end{proof}

\section{Proofs for Section \ref{sec:general}}
\begin{proof}[Proof of Lemma~\ref{lemma:bound}]
	For some $n>0$, assume \[Y = \{(s_1,a_1,y_1),\ldots, (s_n,a_n,y_n)\}.\] We will show that $n$ is upper bounded by Eluder dimension. When we add $(s_j, a_j, y_j)$ to $Y$ at Line \ref{g-error}, 
	\begin{enumerate}
		\item The condition at Line \ref{line:check} must be True i.e. from Equation \eqref{cons:1}, there exists $f_1,f_2 \in F$ such that $|f_{1}(s_j, a_j) - f_{2}(s_j, a_j)| > \frac{\rho}{2} - \delta$.
		\item Observe that for any subsequence $B\subset \{(s_1, a_1), \ldots, (s_{j-1}, a_{j-1})\}$ where $(s_j,a_j)$ is $(\frac{\rho}{2} - \delta)$-dependent on $B$ (Definition \ref{def:ind}), \begin{equation}
		\sum_{(s,a)\in B} |f_{1}(s,a) - f_{2}(s,a)|^2 \geq (\frac{\rho}{2} - \delta)^2.
		\end{equation}
		\item Therefore, if there are $K$ disjoint subsequences in $\{(s_1, a_1), \ldots, (s_{j-1}, a_{j-1})\}$ such that $(s_j,a_j)$ is $(\frac{\rho}{2} - \delta)$-dependent on all of them, then \begin{equation}
		\label{eq:lower-bound}
		\sum_{i=1}^{j-1}  |f_{1}(s_i,a_i) - f_{2}(s_i,a_i)|^2 \geq K (\frac{\rho}{2} - \delta)^2.
		\end{equation}
		\item However, using Equation \ref{cons:2}, we have that\begin{equation}
		\label{eq:upper-bound}
		\sum_{i=1}^{j-1}  |f_{1}(s_i,a_i) - f_{2}(s_i,a_i)|^2 \leq (j-1) (2\delta)^2.
		\end{equation}
	\end{enumerate}
	
	Therefore, we can upper bound for any state-action pair $(s_j,a_j)\in \{(s_1,a_1),\ldots, (s_n,a_n)\}$, the number of disjoint subsequences $K$ in $\{(s_1, a_1), \ldots, (s_{j-1}, a_{j-1})\}$ that $(s_j,a_j)$ is $(\frac{\rho}{2} - \delta)$-dependent on, i.e. \[
	K \leq  \frac{(j-1) (2\delta)^2}{(\frac{\rho}{2} - \delta)^2}.
	\]
	Moreover, it follows from the proof of Proposition 3 in~\citep{russo2013eluder} that for any sequence of state-action pairs say $\{(s_1,a_1),\ldots, (s_n,a_n)\}$, there exists a $(s_j,a_j)$ which is $(\frac{\rho}{2} - \delta)$-dependent on at least  $\frac{n}{\dim_E(F, \frac{\rho}{2} - \delta)} - 1$ disjoint subsequences in $\{(s_1,a_1),\ldots, (s_{j-1},a_{j-1})\}$. Therefore, \begin{equation}
	\frac{n}{\dim_E(F, \frac{\rho}{2} - \delta)} - 1 \leq K \leq \frac{(j-1) (2\delta)^2}{(\frac{\rho}{2} - \delta)^2}
	\end{equation} and thus
	\begin{equation}
	n \leq \dim_E(F, \frac{\rho}{2} - \delta) \left(\frac{(n-1) (2\delta)^2}{(\frac{\rho}{2} - \delta)^2} + 1\right).
	\end{equation} As $\rho>4\delta$, we get \begin{equation}
	\label{eq:n-condition}
	n \leq \dim_E(F, \frac{\rho}{4}) \left(\frac{(n-1) (2\delta)^2}{(\frac{\rho}{2} - \delta)^2} + 1\right)
	\end{equation} which follows from definition of Eluder dimension since $a<b$ implies $\dim_E(F,a)\geq \dim_E(F,b)$. For any $\rho$ and $c>1$ such that \begin{equation}
	\label{eq:rho-condition}
	\rho \geq 2\left(2\sqrt{\frac{c\dim_E(F, \frac{\rho}{4})-1}{c-1}} + 1\right) \delta
	\end{equation} we get from Equation~\eqref{eq:n-condition} that
	\begin{equation}
	n\leq c\dim_E(F, \frac{\rho}{4}).
	\end{equation}
\end{proof}

\section{Proof of Proposition~\ref{thm:informal_lb}}
In this section, we briefly discuss how to generalize the results in~\citep{du2019good} to prove Proposition~\ref{thm:informal_lb}.
We first recall Theorem 4.1 in~\citep{du2019good}.

\begin{proposition}[Theorem 4.1 in~\citep{du2019good}]\label{thm:old_lb}
 There exists a family of deterministic systems $\mathcal{M}$ such that for any $M \in \mathcal{M}$, the following conditions hold.
 There exists a feature extractor $\phi : \states \times \actions \to \mathbb{R}^d$ and $\theta_1, \theta_2, \ldots, \theta_H \in \mathbb{R}^d$ such that $d = O(H / \delta^2)$, and for any $h \in [H]$ and any $(s, a) \in \states_h \times \actions$,
 \[
|Q^*(s, a) - \theta_h^{\top} \phi(s, a)| \le \delta.
 \]
 Moreover, for the deterministic systems in $\mathcal{M}$, any algorithm that returns a $1/2$-optimal policy with probability $0.9$ needs to sample $\Omega(2^H)$ trajectories.
\end{proposition}

We first note that the assumption in Proposition~\ref{thm:old_lb} is slightly different from ours.
In this paper, we assume there exists a single vector $\theta \in \mathbb{R}^d$ such that for any $(s, a) \in \states \times \actions$,
 \[
|Q^*(s, a) - \theta^{\top} \phi(s, a)| \le \delta.
 \]
However, the lower bound in~\citep{du2019good} can still be generalized to hold under our assumption, if one breaks the feature space into $H$ blocks so that each block contains $d / H$ coordinates, and for any state $s_1 \in \states_1$ and $a \in \actions$, $\phi(s_1, a)$ contains non-zero entries only in the first block, and for any state $s_2 \in \states_2$ and $a \in \actions$, $\phi(s_2, a)$ contains non-zero entries only in the second block, etc.
By doing so, we need to change the condition $d = O(H / \delta^2)$ to $d = O(H^2 / \delta^2)$.

Moreover, in order to prove an $\Omega(2^C)$ sample complexity lower bound, one only needs to use the first $C$ levels in the family of deterministic systems in Proposition~\ref{thm:old_lb}, and add $H - C$ dummy levels so that there are $H$ levels in total.
In this case, Proposition~\ref{thm:old_lb} requires $d = O(C^2 / \delta^2)$, or equivalently, $\delta = \Omega(C / \sqrt{d})$.

Finally, by scrutinizing the construction in~\citep{du2019good}, it can be seen that the optimality gap $\gapmin = 1$.

\end{document}